\documentclass[runningheads]{llncs}

\usepackage[utf8]{inputenc} 
\usepackage{hyperref}       
\usepackage{url}            
\usepackage{amsmath,amsfonts,amssymb,mathtools}

\usepackage{tikz}
\usetikzlibrary{shapes,arrows,backgrounds,positioning,plotmarks,calc,patterns,plothandlers,decorations.pathmorphing,decorations.pathreplacing,cd,datavisualization}
\usepackage{pgfplots}
\pgfplotsset{compat=1.10}
\usepgfplotslibrary{fillbetween}

\usepackage{todonotes}
\usepackage{enumitem}
\usepackage{wrapfig}
\usepackage{ifthen}
\usepackage{graphicx}

\usepackage{glossaries-extra}
\setabbreviationstyle[acronym]{long-short}

\glsxtrinsertinsidetrue
\newglossaryentry{gp}
{
	name={GP},
	description={Gaussian Process},
	first={Gaussian Process (GP)},
	plural={GPs},
	descriptionplural={Gaussian Processes},
	firstplural={Gaussian Processes (\glsentryplural{gp})}
}

\GlsXtrSetAltModifier{!}{hyper=false}

\newcommand{\N}{{\mathbb{N}}}
\newcommand{\Z}{{\mathbb{Z}}}
\newcommand{\Q}{{\mathbb{Q}}}
\newcommand{\R}{{\mathbb{R}}}
\newcommand{\F}{\ensuremath{\mathcal{F}}}
\newcommand{\G}{{\mathcal{G}}}
\renewcommand{\H}{{\mathcal{H}}}
\newcommand{\I}{{\mathcal{I}}}

\newcommand{\V}{{\mathcal{V}}}

\DeclareMathOperator{\GP}{\mathcal{GP}}
\DeclareMathOperator{\sol}{sol}
\DeclareMathOperator{\Hom}{Hom}

\newcommand{\domain}{D}

\newcommand{\diffalg}{K}

\newcommand{\derivation}{\delta}
\newcommand{\diffprimeideal}{P}
\newcommand{\diffindet}{\partial}
\newcommand{\ringdiffop}{R}

\newcommand{\idealdiffop}{I}
\newcommand{\matrixrepres}{A}

\DeclareMathOperator{\Mon}{Mon}
\DeclareMathOperator{\lm}{lm}
\DeclareMathOperator{\lc}{lc}

\newcommand{\exponentvector}{\varepsilon}
\newcommand{\multiindexone}{\textbf{1}}

\newtheorem{assumption}{Assumption}

\usetikzlibrary{external} 
\tikzexternaldisable

\title{\gls!{gp} Priors for Linear Systems of Differential Equations with Smooth Boundary Conditions}
\author{, }
\date{}

\begin{document}

\title{On boundary conditions parametrized by analytic functions}
%
%
\author{Markus Lange-Hegermann\inst{1} \and
Daniel Robertz\inst{2}}
\authorrunning{Lange-Hegermann and Robertz}
%
\institute{Technische Hochschule Ostwestfalen-Lippe, inIT (Institute Industrial  IT), Lemgo, Germany 
\email{markus.lange-hegermann@th-owl.de}\\
\and
Lehrstuhl f\"ur Algebra und Zahlentheorie, RWTH Aachen University, Germany\\
\email{daniel.robertz@rwth-aachen.de}}
\maketitle              
\begin{abstract}

Computer algebra can answer various questions about partial differential equations using symbolic algorithms.
However, the inclusion of data into equations is rare in computer algebra.
Therefore, recently, computer algebra models have been combined with Gaussian processes, a regression model in machine learning, to describe the behavior of certain differential equations under data.
While it was possible to describe polynomial boundary conditions in this context, we extend these models to analytic boundary conditions.
Additionally, we describe the necessary algorithms for Gr\"obner and Janet bases of Weyl algebras with certain analytic coefficients.
Using these algorithms, we provide examples of divergence-free flow in domains bounded by analytic functions and adapted to observations.

\end{abstract}

\section{Introduction}
\label{sec:intro}

Differential algebra is concerned with structural properties of systems of ordinary and partial differential equations (ODEs and PDEs) and
provides algorithms for their analysis \cite{Ob,bachler2012algorithmic}.
The properties unveiled by these algorithms correspond to intrinsic properties of the solutions of the system.
At the same time these algorithms isolate equations of interest via elimination, transform systems into normal forms \cite{gerdt2019maple}, describe singularities \cite{lange2021singularities},
allow to investigate control-theoretic properties \cite{lange2020thomas,lange2013thomas}, or detect the size of solution sets \cite{lange2018differential,lange2017differential,lange2014counting}.

Usually, PDEs come with additional information on the evaluation of functions.
For example in inverse problems, parameters in differential equations are being estimated from data points.
Or in theoretical and numerical methods for PDEs, boundary conditions, i.e.\ evaluations of functions on manifolds, ensure well-posedness.
Data points and boundary conditions have rarely been addressed by algebraic means, with the exception of modeling of boundary conditions by in\-te\-gro-diffe\-ren\-tial operators \cite{regensburger2009skew,rosenkranz2008solving}. 

Seemingly disconnected from these algebraic algorithms are \glspl!{gp} \cite{RW}, a general regression technique, which arise as limit of large neural networks \cite{neal1996priors} 
and generalize linear (ridge) regression
, Kriging
, and many spline models
.
\glspl!{gp} describe probability distributions on function spaces.
As such,
\begin{enumerate}[label=(\arabic*)]
    \item\label{item:GPfeatures1} they can be conditioned on observations given as data points using Bayes' rule in closed form, which avoids overfitting,
    \item\label{item:GPfeatures2} they admit an extensive dictionary between their mathematical properties and their covariance function, which allows to prescribe intended behavior,
    \item\label{item:GPfeatures3} form the maximum entropy prior distribution under the assumption of a finite mean and variance in the unknown behavior, and
    \item\label{item:GPfeatures4} the class of \glspl!{gp} is closed under various operations like conditioning, marginalization, and linear operators.
\end{enumerate}
They are typically used in applications when data is rare or expensive to produce, e.g., in active learning \cite{zimmer2018safe}, biology \cite{LawrencePNAS2015}, anomaly detection \cite{BernsTowards2020} or engineering \cite{TLRB_gp}.
The mean function of the posterior is used for regression and the variance quantifies uncertainty.
In that sense, they allow to deal with data, noise, and uncertainty in a way algebraic algorithms usually cannot.

The inclusion of algebraic methods for differential equations into covariance functions of \glspl!{gp} began by divergence-free and curl-free vector fields \cite{MacedoCastro2008,scheuerer2012covariance}, extended to electromagnetic fields \cite{Wahlstrom13modelingmagnetic,MagneticFieldGP} and strain fields \cite{jidling2018probabilistic}.
These approaches were formalized in \cite{LinearlyConstrainedGP}, building on \cite{sarkka2011linear}.
Then, \cite{LH_AlgorithmicLinearlyConstrainedGaussianProcesses} used Gr\"obner bases and worked out the necessity of systems being controllable.
Boundary conditions were added to the setup in \cite{LH_AlgorithmicLinearlyConstrainedGaussianProcessesBoundaryConditions}, restricted to simple polynomial boundaries.

In this paper, we develop algebraic algorithms suitable for this framework to deal with analytic boundary conditions.
These algorithms might take
\begin{enumerate}[label=(\roman*)]
    \item parametrizable linear systems of differential equations,
    \item assumptions on the solutions of the differential equations, e.g.\ smoothness,
    \item various forms of boundary conditions specified by analytic functions, and
    \item (noisy or noiseless) evaluations of functions at finitely many points
\end{enumerate}
as inputs.
They yield a probability distribution on the solution space of the differential equation given by a \gls!{gp}, which has the above properties \ref{item:GPfeatures1}--\ref{item:GPfeatures4}.


Our approach is as follows.
We construct a first parametrization of the solution set of the system of differential equations by finding a matrix whose row nullspace is generated by the equations of the given system.
We take a second parametrization of the boundary condition.
Then, we construct a parametrization of the intersection of the images of these two parametrizations.
Algorithmically, this requires Gr\"obner bases over a Weyl algebra enlarged by various analytic functions, for which we develop the necessary theory and algorithms.
After this symbolic approach, numeric algorithms incorporate measurement data into the \gls!{gp}.

In this setup, ODEs are trivial, both algebraically, as parametrizable linear systems of ODEs with constant or variable coefficients are isomorphic to free systems due to the Jacobson form \cite{TheoryOfRings}, and also from the stochastic point of view, as boundary conditions in ODEs can be modelled by conditioning on data points \cite{GOODE}.
Hence, we focus on PDEs.

From the point of view of machine learning, the results of this paper allow to incorporate information into the covariance structure of a \gls!{gp} prior.
This prior is supported by solutions of the differential equation and the boundary conditions.
In particular, rare measurement data can refine and improve this prior knowledge, instead of being necessary to learn this prior knowledge.

The contributions of this paper can be summarized as follows:
\begin{enumerate}[label=(\alph*)]
    \item we develop Gr\"obner basis algorithms for Weyl algebras over certain rings of analytic functions (cf.\ Sects.~\ref{section_Ore} and \ref{section_modules}),
    \item we study boundary conditions parametrized by analytic functions, in particular how they constrain \glspl!{gp} (cf.\ Sect.~\ref{section_boundary}), and
    \item we construct \gls!{gp} priors for solution sets of PDEs including boundary conditions (cf.\ Sect.~\ref{section_examples}).
\end{enumerate}

\section{Gaussian processes}\label{section_GP}

A \glsreset{gp}\emph{\gls!{gp}} $g=\GP(\mu,k)$ defines a probability distribution on the evaluations of functions $\domain\to\R^\ell$ where $\domain\subseteq\R^d\equiv\R^{1\times d}$ such that function values $g(x_1),\ldots,g(x_n)$
at points $x_1,\ldots,x_n\in\domain$ are jointly (multivariate) Gaussian.
A \gls!{gp} $g$ is specified by a \emph{mean function}
    $\mu:\domain\to\R^\ell:x\mapsto E(g(x))$
and a positive semidefinite\footnote{The function $k$ is positive (semi)definite if and only if for any $x_1,\ldots,x_n\in\domain$ the matrix $K=(k(x_i,x_j))_{i,j}\in\R^{n\ell\times n\ell}$ is positive (semi)definite, i.e.\ $K \succeq 0$.} \emph{covariance function}
\begin{align*}
    k: \domain\times\domain \longrightarrow \R^{\ell\times\ell}_{\succeq0}: (x,x') \longmapsto E\left((g(x)-\mu(x))(g(x')-\mu(x'))^T\right)\mbox{.}
\end{align*}
Any finite set of evaluations of $g$ follows the multivariate Gaussian distribution
\begin{align*}
  \begin{bmatrix}g(x_1)\\ \vdots \\ g(x_n)\end{bmatrix}
  \sim
  \mathcal{N}\left(\begin{bmatrix}\mu(x_1)\\ \vdots \\ \mu(x_n)\end{bmatrix},\begin{bmatrix}
      k(x_1,x_1) & \ldots & k(x_1,x_n) \\
      \vdots & \ddots & \vdots \\
      k(x_n,x_1) & \ldots & k(x_n,x_n)
  \end{bmatrix}\right).
\end{align*}
Now, one knows where a function value $g(x)$ is supposed to be (mean $\mu(x)$), which ignorance we have about $g(x)$ (variance $k(x,x)$), and how two function values $g(x_1)$ and $g(x_2)$ are related (covariance $k(x_1,x_2)$).
%
\glspl!{gp} are popular functional priors in Bayesian inference due to their maximum entropy property \cite{jaynes2003probability}. 

Assume the probabilistic regression model $y=g(x)$ for a \gls!{gp} $g=\GP(0,k)$.
Normalizing the data to mean zero justifies assuming a prior mean function zero. 
Conditioning the \gls!{gp} on training data points 
    $(x_i,y_i)\in \domain\times\R^{1\times \ell}$
for $i=1,\ldots,n$ by Bayes' theorem yields the posterior
\begin{align*}
    p(\;g(x)=y\;|\;g(x_i)=y_i\;)=\frac{p(\;g(x_i)=y_i\;|\;g(x)=y\;)}{p(\;g(x_i)=y_i\;)}\cdot p(\;g(x)=y\;)\mbox{,}
\end{align*}
where $i$ always runs from $1$ to $n$.
All of these distributions are multivariate Gaussian.
Hence, the posterior $p(\;g(x)=y\;|\;g(x_i)=y_i\;)$ is again a \gls!{gp} and can be computed in closed form via linear algebra:
\begin{align}\label{eq_gp_posterior}
	\GP\Big( \quad \quad \quad x&\mapsto yk(X,X)^{-1}k(X,x), \\
	(x,x')&\mapsto k(x,x')-k(x,X)k(X,X)^{-1}k(X,x')\quad \Big)\nonumber\mbox{,}
\end{align}
where $y\in \R^{1\times \ell n}$ denotes the row vector obtained by concatenating the $y_i$ and $k(x,X)\in \R^{\ell\times\ell n}$ resp.\ $k(X,x)\in \R^{\ell n\times\ell}$ resp.\ $k(X,X)\in \R^{\ell n\times\ell n}_{\succeq0}$ denote the (covariance) matrices obtained by concatenating the blocks $k(x,x_j)$ resp.\ $k(x_j,x)$ resp.\ $k(x_i,x_j)$ to a matrix.
In case of a noisy data $(y_i)_j$, one adds the \emph{noise variance} $var((y_i)_j)$ to the $((i-1)\ell+j)$-th diagonal entry of $k(X,X)$.
The Cholesky decomposition improves numerical stability regarding the inversion of the positive definite matrix $k(X,X)$ \cite{RW}.
In the posterior \eqref{eq_gp_posterior}, the mean function can be used as regression model and its variance as model uncertainty.

The class of \glspl!{gp} is closed under linear operators once mild assumptions hold, e.g.\ the derivative of a \gls!{gp} with differentiable realizations is again a \gls!{gp}.

\begin{figure}
  \centering
  \begin{tikzpicture}[scale=0.4]
    \begin{axis}[ymin=-2, ymax=2,xmin=-3,xmax=3]
    \addplot[name path=upper,domain=-2.5:2.5,blue] {-.990427971851666*exp(-(1/2)*(x+2)^2)+.990427971851666*exp(-(1/2)*(x-2)^2)+2-1.98019823825307*exp(-(x+2)^2)+0.131541090052936e-2*exp(-x^2-4)-1.98019823825307*exp(-(x-2)^2)};
    \addplot[name path=lower,domain=-2.5:2.5,blue] {-.990427971851666*exp(-(1/2)*(x+2)^2)+.990427971851666*exp(-(1/2)*(x-2)^2)-2+1.98019823825307*exp(-(x+2)^2)-0.131541090052936e-2*exp(-x^2-4)+1.98019823825307*exp(-(x-2)^2)};
    \addplot [color=blue,fill=blue,fill opacity=0.05] fill between[of=lower and upper, soft clip={domain=-2.5:2.5} ];
    \addplot[name path=mean,domain=-2.5:2.5,blue]{-.990427971851666*exp(-(1/2)*(x+2)^2)+.990427971851666*exp(-(1/2)*(x-2)^2)};
    \addplot[mark=*] coordinates {(-2,-1)};
    \addplot[mark=*] coordinates {(2,1)};
    \end{axis}
  \end{tikzpicture}
  \begin{tikzpicture}[scale=0.4]
    \begin{axis}[ymin=-2, ymax=2,xmin=-3,xmax=3]
    
    \addplot[name path=upper,domain=-2.5:2.5,blue] {(0.907655997014386e-1-0.197316510985624e-1*x^2+0.526177478930923e-2*x)*exp(-x^2-4)+(-.995738707600120+.993745428529554*x)*exp(-(1/2)*(x-2)^2)+(1.00098648143689+.996367574180994*x)*exp(-(1/2)*(x+2)^2)+(-9.90122887197933+7.92100092363468*x-1.98024367721336*x^2)*exp(-(x-2)^2)+2.+(-9.90120090954602-7.92100441893886*x-1.98025066782170*x^2)*exp(-(x+2)^2)};
    \addplot[name path=lower,domain=-2.5:2.5,blue] {(-0.907655997014386e-1+0.197316510985624e-1*x^2-0.526177478930923e-2*x)*exp(-x^2-4)+(-.995738707600120+.993745428529554*x)*exp(-(1/2)*(x-2)^2)+(1.00098648143689+.996367574180994*x)*exp(-(1/2)*(x+2)^2)+(9.90122887197933-7.92100092363468*x+1.98024367721336*x^2)*exp(-(x-2)^2)-2.+(9.90120090954602+7.92100441893886*x+1.98025066782170*x^2)*exp(-(x+2)^2)};
    \addplot [color=blue,fill=blue,fill opacity=0.05] fill between[of=lower and upper, soft clip={domain=-2.5:2.5} ];
    \addplot[name path=mean,domain=-2.5:2.5,blue]{-.991748666925102*exp(-(1/2)*(x+2)^2)+.991752149458988*exp(-(1/2)*(x-2)^2)+(.996367574180994*(x+2))*exp(-(1/2)*(x+2)^2)+(.993745428529554*(x-2))*exp(-(1/2)*(x-2)^2)};
    \addplot[mark=*] coordinates {(-2,-1)};
    \addplot[name path=d1,domain=-2.1:-1.9,black,line width=2]{x+1};
    \addplot[mark=*] coordinates {(2,1)};
    \addplot[name path=d2,domain=1.9:2.1,black,line width=2]{x-1};
    \end{axis}
  \end{tikzpicture}


  \caption{Left: a regression plot (mean and the $2\sigma$ confidence bands) of a \gls!{gp} with mean zero and squared exponential covariance function conditioned on the points $(-2,-1)$ and $(2,1)$ with noise variance of $0.1^2$.
  Right: the \gls!{gp} is additionally conditioned on derivative $1$ with noise $0.1^2$ at both data points.}
  \label{figure_comparison_regression_derivative}
\end{figure}
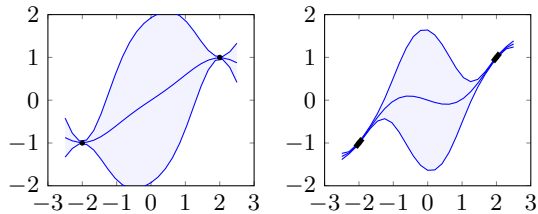

\begin{wrapfigure}{r}{2.0cm}
	\vspace{-1.5em}
	\begin{tikzcd}
	D\arrow[rd, bend left, "b\circ g\in b_*G"]\arrow[d, "g\in G"] \\
	Y\arrow[r,"b"] & Z
	\end{tikzcd}
	\!\!\!\!\!\!
\end{wrapfigure}
Given a set of functions $G\subseteq Y^D$ and $b:Y\to Z$, then the \emph{pushforward} is
$
b_*G=\{b\circ f\mid f\in G\}\subseteq Z^D
$.
The \emph{pushforward} of a stochastic process $g:\domain\to Y$ by $b:Y\to Z$ is defined as
\begin{align*}
b_*g:\domain\longrightarrow Z: d\longmapsto(b\circ g)(d).
\end{align*}

\begin{lemma}[{\cite[Lemma~2.2]{LH_AlgorithmicLinearlyConstrainedGaussianProcessesBoundaryConditions}}]\label{lemma_pushforward_gaussian}
	Let $\F$ and $\G$ be spaces of functions defined on a set $\domain$ with product $\sigma$-algebra of the function evaluations.
	Let $g=\GP(\mu(x),k(x,x'))$ with realizations in $\F$ and $B:\F\to\G$ a linear, measurable operator which commutes with expectation w.r.t.\ the measure induced by $g$ on $\F$ and by $B_*g$ on $\G$.
	Then, the \glsreset{gp}\emph{pushforward \gls!{gp}} $B_*g$ of $g$ under $B$ is a \gls!{gp} with
	\begin{align*}
	B_*g = \GP(B\mu(x),Bk(x,x')(B')^T)\mbox{ ,}
	\end{align*}
	where $B'$ denotes the operation of $B$ on functions with argument $x'$.
\end{lemma}

\begin{example}\label{example_derivative_GP}
	Let $g=\GP(0,k(x,x'))$ be a \gls!{gp} with realizations (a.s.) in the set $C^1(\R,\R)$ of differentiable functions.
	The pushforward \gls!{gp}
	\begin{align*} 
		\begin{bmatrix}\frac{\partial}{\partial x}\end{bmatrix}_*g := \GP\left(0,\frac{\partial^2}{\partial x\partial x'}k(x,x')\right)\mbox{}
	\end{align*}
	describes derivatives of the \gls!{gp} $g$ \cite[\textsection5.2]{StationaryAndRelatedStochasticProcesses}.
	The one-argument derivative $\frac{\partial}{\partial x}k(x,x')$ yields the cross-covariance between on the one hand a function evaluation $g(x')$ of $g$ at $x'\in\R$ and on the other hand its derivative $(\begin{bmatrix}\frac{\partial}{\partial x}\end{bmatrix}_*g)(x)$ evaluated at $x\in\R$.
	We use this to include data of derivatives into a model in Figure~\ref{figure_comparison_regression_derivative}.
\end{example}

\section{Solution sets of operator equations}\label{section_GPs_solutions}

This section discusses how \glspl!{gp} describe the real vector space $\F=C^\infty(\domain,\R)$, a candidate set of solutions for the linear differential equations, and how such \glspl!{gp} interplay with linear operators.
Assume that $\domain\subset\R^d$ is compact and $\F$ is endowed with the usual Fr\'echet topology generated by the separating family
\begin{align}\label{eq_frechet}
	\|f\|_{a}:=
	\sup_{\substack{i\in\Z_{\ge0}^d\\ |i|\le a}}
	\sup_{x\in D}\ \left|\frac{\partial^{|i|}}{\partial x^i}f(x)\right|
\end{align}
of seminorms for all $a\in\Z_{\ge0}$%
, where $i=(i_1,\ldots,i_d)\in\Z_{\ge0}^d$ is a multi-index with $|i|=i_1+\ldots+i_d$.
The \emph{squared exponential covariance function}
\begin{align}\label{eq_SE}
k_\F:\R^d\times\R^d\longrightarrow\R:(x_i,x_j)\longmapsto\exp\left(-\frac{1}{2}\sum_{a=1}^d(x_{i,a}-x_{j,a})^2\right)
\end{align}
induces an adapted \gls!{gp} prior in $\F=C^\infty(\domain,\R)$.

\begin{proposition}
    The scalar \gls!{gp} $g_\F=\GP(0,k_\F)$ has realizations dense (a.s.) in $\F$ with respect to the Fr\'echet topology defined by Equation \eqref{eq_frechet}.
\end{proposition}
\begin{proof}
	We show that the realizations of $g_\F$ are densely contained in $\F$ in three steps: first, the realizations are contained in $\F$, i.e.\ smooth; second, the elements of the reproducing kernel Hilbert space (RKHS)\footnote{%
		For $\GP(0,k)$, the set $\H^0(g)$ generated as a vector space by the $x\mapsto k(x_i,x)$ for $x_i\in D$ with scalar product
		$\left\langle k(x_i,-),k(x_j,-) \right\rangle := k(x_i,x_j)$
		is a pre-Hilbert space.
		Its closure $\H(g)$ is the \emph{reproducing kernel Hilbert space} of the \gls!{gp} $g$ \cite{RKHSProbabilityStatistics}.
    } $\H(g_\F)$ of the \gls!{gp} $g_\F$, are realizations; and third, the RKHS $\H(g_\F)$ is dense in $\F$.
    
    First, show that the realizations of $g_\F$ lie in $\F$ . 
    They are continuously differentiable, as $k_\F$ is twice continuously differentiable \cite[(9.2.2)]{StationaryAndRelatedStochasticProcesses}.
    Continue inductively, as the covariance $\frac{\partial^2}{\partial x\partial x'}k_\F(x,x')$ of the derivative of $g_\F$ is again smooth. 
    
	
	For the second step, we note that $C^\infty(\domain,\R)$ is Radon as $\domain$ is compact, hence $g_\F$ induces a Radon measure on $\F$.
	For any Radon measure, $\H(g_\F)$ is contained in the topological support of the measure induced by $g_\F$ by \cite[Thm.~3.6.1]{bogachev1998gaussian}.
	For this, $\F=C^\infty(\domain,\R)$ needs to be locally convex, which it is being Fr\'echet.
	
  	For the third step, by \cite[Prop.~4]{SimSch16}, $\H(g_\F)$ is continuously contained
  	in $\F$ and dense by \cite[Thm.~12, Prop.~42]{SimSch16} or \cite[after proof of Cor.~38]{SimSch16}.\qed
\end{proof}

The following three $\R$-algebras $R$ model linear operator equations by making $\F$ a left $R$-module.
Sects.~\ref{section_Ore} and \ref{section_modules} introduce Gr\"obner bases for such rings.

\begin{example}\label{example_linear_dgl}
	The polynomial ring $R=\R[\partial_{x_1},\ldots,\partial_{x_d}]$ models linear PDEs with constant coefficients, where $\partial_{x_i}$ acts on $\F=C^\infty(D,\R)$ via partial derivative with respect to $x_i$.
\end{example}

\begin{example}\label{example_polynomial_operator_ring}
    Let $f_1,\ldots,f_n\in \F$ be functions.
	The ring $R=\R[f_1,\ldots,f_n]$ is commutative and models boundary conditions by multiplication, see Sect.~\ref{section_boundary}.
\end{example}

\begin{example}\label{example_Weyl}
    Let $F\subseteq \F$ be an $\R$-algebra closed under partial derivatives.
	To combine linear differential equations with boundary conditions, consider the Weyl algebra $R=\R[F]\langle \partial_{x_1},\ldots,\partial_{x_d}\rangle$.
	The non-commutative relation $\partial_{x_i}f=f\partial_{x_i}+\frac{\partial f}{\partial x_i}$ represents the product rule of differentiation for $f\in F$ and $1\le i\le d$.
\end{example}

Operators defined over these three rings satisfy the assumptions of Lemma~\ref{lemma_pushforward_gaussian}:
multiplication commutes with expectations and the dominated convergence theorem implies that expectation commutes with derivatives, as realizations of $g_\F$ are continuously differentiable.
Furthermore, these rings act continuously on $\F$:
the Fr\'echet topology makes derivation continuous by construction, and multiplication by elements in $\F$ is bounded as $D$ is compact, which implies continuity in the Fr\'echet space $\F$.
In particular, we have the following:

\begin{corollary}\label{coro_pushforward}
	Let $\F=C^\infty(D,\R)$ be the space of smooth functions defined on a compact set $\domain\subset\R^d$.
	Let $g=\GP(\mu(x),k(x,x'))$ with realizations in $\F^{\ell''}$ and $B:\F^{\ell''}\to\F^\ell$ a linear operator over one of the operator rings in Examples~\ref{example_linear_dgl}, \ref{example_polynomial_operator_ring}, or \ref{example_Weyl}.
	Then, the \emph{pushforward \gls!{gp}} $B_*g$ is again Gaussian with
	\begin{align*}
	    B_*g = \GP(B\mu(x),Bk(x,x')(B')^T)\mbox{ ,}
	\end{align*}
	where $B'$ denotes the operation of $B$ on functions with argument $x'$.
\end{corollary}

\section{Parametrizations}\label{section_paramerization}

We consider solution sets of linear differential equations, how to parametrize them by a suitable matrix $B$ and thereby describe them by a \gls!{gp} $B_*g$.
Let $\ringdiffop$ be one of the rings from the previous section, $\F$ the left $\ringdiffop$-module $C^{\infty}(\domain, \R)$ and
$A\in R^{\ell'\times\ell}$.
Define the \emph{solution set} $\sol_\F(A):=\{f\in \F^{\ell\times1}\mid Af=0\}$ of $A$.
We say that a \gls!{gp} is \emph{in} a function space, if its realizations are a.s.\ contained in said space.
We first describe the interplay of \glspl!{gp} and solution sets of operators.

\begin{lemma}[{\cite[Lemma~2.2]{LH_AlgorithmicLinearlyConstrainedGaussianProcesses}}]\label{lemma_gp_operator}
	Let $g=\GP(\mu,k)$ be a \gls!{gp} in $\F^{\ell\times1}$.
	Then $g$ is a \gls!{gp} in the solution set $\sol_\F(A)$ of $A\in R^{\ell'\times\ell}$ if and only if both $\mu$ is contained in $\sol_\F(A)$ and $A_*(g-\mu)$ is the constant zero process.
\end{lemma}

This lemma motivates how to construct \glspl!{gp} with realizations in $\sol_\F(A)$: find a $B\in R^{\ell\times\ell''}$ with $AB=0$ \cite{LinearlyConstrainedGP}.
Then, taking any \gls!{gp} $g=\GP(0,k)$ in $\F^{\ell''\times 1}$, the realizations of $B_*g$ are (possibly strictly) contained in $\sol_\F(A)$, as $A_*(B_*g)=(AB)_*g=0_*g=0$.
One prefers to enlarge $B$ to approximate \emph{all} solutions in $\sol_\F(A)$ by $B_*g$, i.e., the realizations of $B_*g$ should be dense in $\sol_\F(A)$.
Call $B\in R^{\ell\times\ell''}$ a \emph{parametrization} of $\sol_\F(A)$ if $\sol_\F(A)=B\F^{\ell''\times 1}$.
Such a parametrization does not always exist, e.g., for the matrix $A=\begin{bmatrix}\partial_{x_1}\end{bmatrix}$.

\begin{proposition}[{\cite[Proposition~3.5]{LH_AlgorithmicLinearlyConstrainedGaussianProcessesBoundaryConditions}}]\label{proposition_denseparametrization}
	Let $B\in R^{\ell\times\ell''}$ be a parametrization of $\sol_\F(A)$.
	Let $g_\F^{\ell''\times 1}$ be the \gls!{gp} of $\ell''$ i.i.d.\ copies of $g_\F$, the \gls!{gp} with squared exponential covariance $k_\F$ \eqref{eq_SE}.
	Then, $B_*g_\F^{\ell''\times 1}$ has realizations dense in $\sol_\F(A)$.
\end{proposition}

We summarize how to algorithmically decide whether a parametrization exists and how to compute it in the positive case.
Computations directly over the space of functions $\F$ are infeasible.
Hence, we compute over $R$ instead.
Inferring results over $\F$ is possible once $\F$ is an injective\footnote{%
	In algebraic system theory, one usually works with injective cogenerators $\F$ \cite{Q_habil}.
	Injective cogenerators allow to infer back from analysis in $\F$ to algebra over $\ringdiffop$.
	In our setting, this step back is superfluous, as the algebra cannot encode data points.
} $R$-module, i.e.\ $\Hom_R(-,\F)$ is exact.
Luckily, for PDEs with constant coefficients we have the following:

\begin{theorem}[{\cite{malgrange1956existence}, \cite{ehrenpreis1954solution} \cite[\textsection2(54)]{Ob}, }]\label{theorem_injective_module}
    Let $R=\R[\partial_{x_1},\ldots,\partial_{x_d}]$ be as in Example~\ref{example_linear_dgl} and $\domain\subset\R^d$ convex.
    Then, $\F=C^\infty(\domain,\R)$ is an injective $R$-module.
\end{theorem}

With this in mind, we recall the construction of parametrizations.

\begin{theorem}[{%
\cite[Thm.~2]{ZSHinverseSyzygies},
\cite[\S7.(24)]{Ob},
\cite{QGrade},
\cite{Q_habil},
\cite{CQR05_nonote},
\cite{RobRecentProgress}%
}]\label{theorem_parametrizable}
    Let $\ringdiffop$ be a ring and $\F$ an injective left $\ringdiffop$-module.
	Let $A\in R^{\ell'\times\ell}$.
	Let $B$ be the right nullspace of $A$ and $A'$ the left nullspace of $B$.
	Then $\sol_\F(A')$ is the largest subset of $\sol_\F(A)$ that is parametrizable, $B$ parametrizes $\sol_\F(A')$, and $\sol_\F(A)$ is parametrizable if and only if the rows of $A$ and $A'$ generate the same row module, i.e.\ if all rows of $A'$ are contained in the row module generated by $A$.
\end{theorem}

Gr\"obner bases
turn Theorem~\ref{theorem_parametrizable} effective, as they allow to compute the right nullspace $B$ of $A$, the left nullspace $A'$ of $B$ and decide whether the rows of $A'$ are contained in the row space of $A$ over $R$.
We have the following criterion.

\begin{theorem}[{\cite[\S7.(21)]{Ob}}]\label{theorem_controllable_parametrizable}
	A system $\sol_\F(A)$ is parametrizable if and only if it is controllable.
    If $A$ is not parametrizable, then the solution set $\sol_\F(A')$ is the subset of controllable behaviors in $\sol_\F(A)$, where $A'$ is defined as in Theorem~\ref{theorem_parametrizable}.
\end{theorem}

Solution sets of differential equations and polynomial boundary conditions can be intersected \cite{LH_AlgorithmicLinearlyConstrainedGaussianProcessesBoundaryConditions}.

\begin{theorem}[{\cite[Theorem~5.2]{LH_AlgorithmicLinearlyConstrainedGaussianProcessesBoundaryConditions}}]\label{theorem_combining_parametrizations}
	Let $B_1\in R^{\ell\times\ell_1''}$ and $B_2\in R^{\ell\times\ell_2'}$.
	Denote by 
	$C:=\begin{bmatrix} C_1 \\ C_2\end{bmatrix}\in R^{(\ell_1'+\ell_2')\times m}$
	the right-nullspace of the matrix $B:=\begin{bmatrix} B_1 & B_2\end{bmatrix}\in R^{\ell\times(\ell_1'+\ell_2')}$.
	Then $B_1C_1=-B_2C_2$ parametrizes solutions of $B_1\F^{\ell_1'}\cap B_2\F^{\ell_2'}$.
\end{theorem}

Here, $B_1$ might be a matrix of differential operators and $B_2$ a matrix of polynomial functions, and we consider both matrices over a common ring $R$.

\section{Rings of differential operators over differential algebras}\label{section_Ore}

We have considered parametrizations by differential operators and in Sect.~\ref{section_boundary} we consider parametrizations of boundary conditions by analytic functions.
For their combination in Sect.~\ref{section_examples}, we now extend classical Gr\"obner and Janet bases. 

Let $\domain\subset\R^d$ be connected and denote by $\derivation_1, \ldots, \derivation_d$ the commuting derivations in the coordinate directions of $\R^d$.
Let $\diffalg$ be a differential algebra over the real numbers\footnote{%
Of course, the constructions in this and the following section work over any sufficiently algorithmic differential field of characteristic zero, not only $\R$. In practice, we assume to work over a computable subfield of $\R$.
} $\R$ generated by analytic function $f_1,\ldots,f_r\colon \domain\to\R$.
For algorithmic reasons assume that $\diffalg$ is finitely presented as a differential algebra over $\R$ as
\[
\diffalg \, = \, \R\{ f_1, \ldots, f_r \} \, \cong \, \R\{ F_1, \ldots, F_r \} / \diffprimeideal,
\]
where $\diffprimeideal$ is a prime differential ideal of $\R\{ F_1, \ldots, F_r \}$, generated by
\begin{equation}\label{eq:Pgenset}
\{ \, \derivation_j F_i - g_{i,j} \mid i = 1, \ldots, r, \, j = 1, \ldots, d \, \},
\end{equation}
where $g_{i,j} \in \R[F_1, \ldots, F_r]$ are (non-differential) polynomials in $F_1, \ldots, F_r$,
and the above isomorphism is given by $f_i\mapsto F_i + \diffprimeideal$.
In particular, the generators $f_1, \ldots, f_r$ of $\diffalg$ are algebraically independent over $\R$. Then $\diffalg$ is isomorphic to
$\R[f_1, \ldots, f_r]$ as an $\R$-algebra, and $\diffalg$ is Noetherian, factorial, and a GCD domain.


\begin{example}\label{ex:funcderiv2}
For the differential algebra $\diffalg = \Q\{ x, y, \exp(x^2+y^2-1) \}$
with derivations $\derivation_1 = \partial/\partial x$,
$\derivation_2 = \partial/\partial y$ we have
$\diffalg \cong \Q\{ F_1, F_2, F_3 \} / P$, where 
\begin{align*}
    \derivation_1 F_1 - 1, && \derivation_2 F_1, &&
    \derivation_1 F_2, && \derivation_2 F_2 - 1, &&
    \derivation_1 F_3 - 2 F_1 F_3, && \derivation_2 F_3 - 2 F_2 F_3,
\end{align*}
generate the prime differential ideal $P$ such that
\begin{align*}
    \Q\{ F_1, F_2, F_3 \} \longrightarrow  \diffalg: F_1 \longmapsto x, \, F_2 \longmapsto y, \, F_3 \longmapsto \exp(x^2+y^2-1)
\end{align*}
is an epimorphism of differential algebras over $\Q$ mapping precisely $\diffprimeideal$ to zero.
\end{example}

\begin{definition}\label{de:R}
Let the \emph{ring of differential operators} $\ringdiffop = \diffalg\langle \diffindet_1, \ldots, \diffindet_d \rangle$
be the iterated Ore extension of $\diffalg$ defined by
\[
\begin{array}{rclll}
\diffindet_i \, a & = & a \, \diffindet_i + \derivation_i(a), \quad
& \quad a \in \diffalg, & \quad i = 1, \ldots, d,\\[0.5em]
\diffindet_i \, \diffindet_j & = & \diffindet_j \, \diffindet_i,
& & \quad i, j = 1, \ldots, d.
\end{array}
\]
\end{definition}

\begin{remark}\label{rem:Oreproperty}
The ring $\ringdiffop$ is (left) Noetherian, because $\diffalg$ is Noetherian
(cf., e.g., \cite[Thm.~1.2.9 (iv)]{McConnellRobson}).
Moreover, $\ringdiffop$ has the left Ore property, i.e., every pair of non-zero elements
of $\ringdiffop$ has a non-zero common left multiple \cite[Thm.~2.1.15]{McConnellRobson}, which, in particular, implies the existence of
a skew field of fractions of $\ringdiffop$.
\end{remark}

We define the set of \emph{monomials} of $\ringdiffop$ as
\[
\Mon(\ringdiffop) = \{ \, f_1^{\alpha_1} \ldots f_r^{\alpha_r} \, \diffindet_1^{\beta_1} \ldots \diffindet_d^{\beta_d} \mid
\alpha_1, \ldots, \alpha_r, \beta_1, \ldots, \beta_d \in \Z_{\ge 0} \, \}.
\]
It is a basis of $\ringdiffop$ as an $\R$-vector space:
every $p \in \ringdiffop$ has a unique representation
\begin{equation}\tag{$*$}\label{eq:repres_element}
p \, = \, \sum_{m \in \Mon(\ringdiffop)} c_m \, m,
\end{equation}
where $c_m \in \R$ and only finitely many $c_m$ are non-zero.

\medskip

A \emph{monomial ordering} $<$ on $\ringdiffop$ is a total ordering on $\Mon(\ringdiffop)$ satisfying
\[
\begin{array}{ll}
f_1^0 \ldots f_r^0 \, \diffindet_1^0 \ldots \diffindet_d^0 = 1 < m &
\quad \mbox{for all } 1 \neq m \in \Mon(\ringdiffop),\\[0.5em]
m_1 < m_2 \quad \Rightarrow \quad f_i \, m_1 < f_i \, m_2 & \quad
\mbox{for all } m_1, m_2 \in \Mon(\ringdiffop), \, i = 1, \ldots, r,\\[0.5em]
m_1 < m_2 \quad \Rightarrow \quad m_1 \, \diffindet_j < m_2 \, \diffindet_j & \quad
\mbox{for all } m_1, m_2 \in \Mon(\ringdiffop), \, j = 1, \ldots, d.
\end{array}
\]
For every $0 \neq p \in \ringdiffop$ the $<$-greatest monomial $m$
occurring with non-zero coefficient $c_m$
in the representation \eqref{eq:repres_element} of $p$ is
called the \emph{leading monomial} of $p$ and is denoted by $\lm(p)$.
Its coefficient $c_m$ is called the \emph{leading coefficient} of $p$
and is denoted by $\lc(p)$.
For a subset $S$ of $\ringdiffop$ we let $\lm(S) = \{ \, \lm(s) \mid 0 \neq s \in S \, \}$.


\begin{example}\label{ex:degrevlex}
The \emph{weighted degree-reverse-lexicographical ordering} $<$ with
weights $w=(w_1, \ldots, w_{r+d})\in\Q_{>0}^{r+d}$ (weighted deg-rev-lex) is defined by
\begin{align*}
f_1^{\alpha_1} \ldots f_r^{\alpha_r} \diffindet_1^{\alpha_{r+1}} \ldots \diffindet_d^{\alpha_{r+d}}
<
f_1^{\alpha'_1} \ldots f_r^{\alpha'_r} \diffindet_1^{\alpha'_{r+1}} \ldots \diffindet_d^{\alpha'_{r+d}} \\
\iff \quad
\Bigg(\sum_{j=i}^{r+d} w_i \alpha_i, \alpha_{r+d}, \ldots, \alpha_1\Bigg)
>_{{\rm lex}}
\Bigg(\sum_{j=i}^{r+d} w_i \alpha'_i, \alpha'_{r+d}, \ldots, \alpha'_1\Bigg),
\end{align*}
where $\alpha_i, \alpha'_i\in \Z_{\ge 0}$ and $>_{{\rm lex}}$
compares tuples lexicographically.
\end{example}

\begin{example}\label{ex:elimpartials}
We let the \emph{elimination ordering} $<$ on $\ringdiffop$
(eliminating $\diffindet_1$, \ldots, $\diffindet_d$) be
\[
\begin{array}{ll}
& f_1^{\alpha_1} \ldots f_r^{\alpha_r} \,
\diffindet_1^{\beta_1} \ldots \diffindet_d^{\beta_d} \, < \,
f_1^{\alpha'_1} \ldots f_r^{\alpha'_r} \,
\diffindet_1^{\beta'_1} \ldots \diffindet_d^{\beta'_d}\\[1em]
\iff & \displaystyle
\Big( \, \diffindet_1^{\beta_1} \ldots \diffindet_d^{\beta_d} \, \prec_{\diffindet} \,
\diffindet_1^{\beta'_1} \ldots \diffindet_d^{\beta'_d} \qquad \mbox{or}\\[1em]
& \displaystyle\quad
\diffindet_1^{\beta_1} \ldots \diffindet_d^{\beta_d} \, = \,
\diffindet_1^{\beta'_1} \ldots \diffindet_d^{\beta'_d} \quad \mbox{and} \quad
f_1^{\alpha_1} \ldots f_r^{\alpha_r} \, \prec_f \,
f_1^{\alpha'_1} \ldots f_r^{\alpha'_r} \, \Big),
\end{array}
\]
where $\alpha_i$, $\alpha'_i$, $\beta_j$, $\beta'_j \in \Z_{\ge 0}$ and where $\prec_{\diffindet}$
and $\prec_f$ are the
deg-rev-lex ordering
on the
polynomial algebras $\Q[\diffindet_1, \ldots, \diffindet_d]$
and $\Q[f_1, \ldots, f_r]$, respectively.
\end{example}

\begin{assumption}\label{assump:termorder}
The monomial
ordering $<$ on $\ringdiffop$ is chosen such that the
leading monomial of
\[
\diffindet_j \, f_i \, = \,
f_i \, \diffindet_j + \delta_j(f_i) \, = \,
f_i \, \diffindet_j + g_{i,j}
\]
with respect to $>$ is $f_i \, \diffindet_j$,
for all $i = 1$, \ldots, $r$ and $j = 1, \ldots, d$.
(Recall that $f_i \, \diffindet_j + g_{i,j}$ is the
representation~\eqref{eq:repres_element} of $\diffindet_j \, f_i$
taking the generators (\ref{eq:Pgenset}) of
the prime differential ideal $\diffprimeideal$ into account.)
\end{assumption}

In what follows, we make Assumption~\ref{assump:termorder},
which is met if $>$ is a degree-reverse-lexicographical
ordering with weights $(v_1, \ldots, v_r, w_1, \ldots, w_d)$ satisfying
\[
w_j \, \ge \, \max_{i=1, \ldots, r} \left( \,
\sum_{k=1}^r v_k \deg_{f_k}(g_{i,j}) - v_i \, \right)
\qquad \mbox{for all } j = 1, \ldots, d,
\]
or if $>$ is an elimination ordering as in Example~\ref{ex:elimpartials}.

\medskip

Before introducing Janet bases for left ideals of $\ringdiffop$
we recall the concept of \emph{Janet division},
which we formulate for ideals of the free commutative semigroup $(\Z_{\ge 0})^{r+d}$
in our context. Note that if $\idealdiffop$ is a non-zero left ideal of $\ringdiffop$,
then the exponent vectors $(\alpha_1, \ldots, \alpha_r, \beta_1, \ldots, \beta_d)$
of all elements of $\lm(\idealdiffop)$ form an ideal of $(\Z_{\ge 0})^{r+d}$
due to the definition of a monomial ordering and Assumption~\ref{assump:termorder}.
The bijection between $\Mon(\ringdiffop)$ and $(\Z_{\ge 0})^{r+d}$ may as well
be chosen to be, e.g.,
\[
\exponentvector\colon
\Mon(\ringdiffop) \longrightarrow (\Z_{\ge 0})^{r+d}:
f_1^{\alpha_1} \ldots f_r^{\alpha_r} \, \diffindet_1^{\beta_1} \ldots \diffindet_d^{\beta_d} \longmapsto (\beta_1, \ldots, \beta_d, \alpha_1, \ldots, \alpha_r),
\]
which is the bijection we usually work with.

\medskip

Recall that every ideal of $(\Z_{\ge 0})^{r+d}$ is finitely generated;
moreover, it has a unique minimal generating set.
For $k \in \{ 1, \ldots, r+d \}$ we denote by $\multiindexone_k$ the
multi-index with $1$ in position $k$ and $0$ elsewhere.
Following M.~Janet (cf., e.g., \cite{Robertz6}) we make the following definition
in terms of exponent vectors.

\begin{definition}
Let $A \subset (\Z_{\ge 0})^{r+d}$ be finite and $\alpha = (\alpha_1, \ldots, \alpha_{r+d}) \in A$.
Then $\exponentvector^{-1}(\multiindexone_k)$ is said to be \emph{multiplicative}
for the monomial $\exponentvector^{-1}(\alpha)$ if and only if
\[
\alpha_k \, = \, \max \{ \, \alpha'_k \mid (\alpha'_1, \ldots, \alpha'_{r+d}) \in A \mbox{ with }
\alpha'_1 = \alpha_1, \, \ldots, \, \alpha'_{k-1} = \alpha_{k-1} \, \}.
\]
Let $M \subset \Mon(\ringdiffop)$ be finite. Then
for every $m \in M$ we obtain a partition $\mu(m, M) \uplus \overline{\mu}(m, M)$ of
$\{ f_1, \ldots, f_r, \diffindet_1, \ldots, \diffindet_d \}$, where each
element of $\mu(m, M)$ is multiplicative for $m$
and each element of $\overline{\mu}(m, M)$ is non-mul\-ti\-pli\-ca\-tive for $m$.
\end{definition}

\begin{example}\label{ex:janetdivision}
Let $r = 2$, $n = 1$,
$M = \{ \, f_1 \, f_2^2, \, f_1^2 \, f_2, \, f_2 \, \diffindet_1^2, \, f_1 \, \diffindet_1^2 \, \}$.
Using the above bijection $\exponentvector$ we obtain
\[
\begin{array}{l}
\mu(f_1 \, f_2^2, M) = \{ \, f_2 \, \}, \qquad
\mu(f_1^2 \, f_2, M) = \{ \, f_1, f_2 \, \},\\[0.5em]
\mu(f_2 \, \diffindet_1^2, M) = \{ \, \diffindet_1, f_2 \, \}, \qquad
\mu(f_1 \, \diffindet_1^2, M) = \{ \, \diffindet_1, f_1, f_2 \, \}.
\end{array}
\]
\end{example}

\begin{definition}\label{de:Janetcomplete}
Let $M \subset \Mon(\ringdiffop)$ be finite. We define two
supersets of $M$ in $\Mon(\ringdiffop)$ as follows:
\[
\begin{array}{rcl}
\langle \, M \, \rangle & = & \displaystyle
\bigcup_{m \in M}
\{ \, f_1^{\phi_1} \ldots f_r^{\phi_r} \, m \,
\diffindet_1^{\psi_1} \ldots \diffindet_d^{\psi_d} \mid
\, \phi_i, \psi_j \in \Z_{\ge 0} \, \},\\[1.5em]
[ \, M \, ] & = & \displaystyle
\biguplus_{m \in M}
\{ \, f_1^{\phi_1} \ldots f_r^{\phi_r} \, m \,
\diffindet_1^{\psi_1} \ldots \diffindet_d^{\psi_d} \mid
\, \phi_i, \psi_j \in \Z_{\ge 0},\\[1.5em]
& & \qquad \qquad \quad
\phi_i = 0 \mbox{ if }
f_i \not\in \mu(m, M) \mbox{ and }
\psi_j = 0 \mbox{ if }
\diffindet_j \not\in \mu(m, M) \, \},
\end{array}
\]
where the latter union is disjoint by construction of Janet division.
The set $M$ of monomials is said to be \emph{Janet complete}
if $[ \, M \, ] = \langle \, M \, \rangle$.
\end{definition}

Any finite subset $M$ of $\Mon(\ringdiffop)$ has a
unique smallest (finite) Janet complete superset of $M$, which we call
the \emph{Janet completion} of $M$ \cite[Subsect.~2.1.1]{Robertz6}.


\begin{definition}\label{de:GB}
Let $\idealdiffop$ be a non-zero left ideal of $\ringdiffop$.
Using the notation of Definition~\ref{de:Janetcomplete},
a finite generating set $G \subset \ringdiffop \setminus \{ 0 \}$
for $\idealdiffop$ is called a \emph{Gr\"obner basis} for $\idealdiffop$ with respect to
the monomial ordering $<$ if
$\langle \, \lm(G) \, \rangle = \lm(\idealdiffop)$.
If moreover, $\lm(G)$ is Janet complete, i.e.,
$[ \, \lm(G) \, ] = \langle \, \lm(G) \, \rangle = \lm(\idealdiffop)$,
then $G$ is called a \emph{Janet basis} for $\idealdiffop$ with respect to $<$.
\end{definition}

Assumption~\ref{assump:termorder} facilitates a multivariate polynomial division in $\ringdiffop$.

\begin{remark}\label{rem:reduction}
Suppose $L \subset \ringdiffop \setminus \{ 0 \}$ is finite and $\lm(L)$ is Janet complete.
Let $p_1 \in \ringdiffop \setminus \{ 0 \}$. If $\lm(p_1) \in [ \, \lm(L) \, ]$, then there
exists a unique $p_2 \in L$ such that
\[
\lm(p_1) \, = \, f_1^{\phi_1} \ldots f_r^{\phi_r} \, \lm(p_2) \,
\diffindet_1^{\psi_1} \ldots \diffindet_d^{\psi_d}
\]
for certain $\phi_i$, $\psi_j \in \Z_{\ge 0}$, where $\phi_i = 0$ if $f_i \not\in \mu(\lm(p_2), \lm(L))$
and $\psi_j = 0$ if $\diffindet_j \not\in \mu(\lm(p_2), \lm(L))$.
Therefore, subtracting $\lc(p_1) \, f_1^{\phi_1} \ldots f_r^{\phi_r} \, \diffindet_1^{\psi_1} \ldots \diffindet_d^{\psi_d} \, p_2$ from $\lc(p_2) \, p_1$ yields either zero or an element
of $\ringdiffop$ whose leading monomial is less than $\lm(p_1)$.
Since a monomial ordering $<$ does not admit infinitely descending chains of monomials,
this reduction procedure always terminates.
\end{remark}

Iterated reduction, as just defined, modulo a Gr\"obner basis
or a Janet basis for the left ideal $\idealdiffop$
allows to decide membership to $\idealdiffop$.

\begin{proposition}\label{prop:membership}
Let $G$ be a Gr\"obner basis or a Janet basis for the left ideal $\idealdiffop$
of $\ringdiffop$ with respect to any monomial ordering $<$, and let $p \in \ringdiffop$.
Then we have $p \in \idealdiffop$ if and only if the
remainder of reduction of $p$ modulo $G$ is zero.
\end{proposition}

\begin{remark}\label{rem:computJB}
Given a finite generating set $L$ for a non-zero left ideal $\idealdiffop$ of $\ringdiffop$
and given a monomial ordering $<$ as above, a
Janet basis for $\idealdiffop$ with respect to $<$ can be computed
in finitely many steps \cite{Robertz6}.
After a preliminary pairwise reduction of elements of $L$ ensuring that the leading monomials of
elements of $L$ are pairwise different and that $\exponentvector(\lm(L))$
is the unique minimal generating set of the ideal of $(\Z_{\ge 0})^{r+d}$ it generates,
multiplicative and non-mul\-ti\-pli\-ca\-tive variables are determined for each leading monomial
(with respect to $\lm(L)$) and $L$ is replaced by its Janet completion.
Reduction of left multiples of elements of $L$ by non-mul\-ti\-pli\-ca\-tive variables may yield
non-zero remainders in $\idealdiffop$. Augmenting $L$ by such elements results in
a larger ideal $\exponentvector(\lm(L))$ of $(\Z_{\ge 0})^{r+d}$ than previously.
Since every ascending chain of such ideals becomes stationary after finitely many steps,
by iteration of these steps, one obtains a generating set $G$ for $\idealdiffop$
whose left multiples by non-mul\-ti\-pli\-ca\-tive variables reduce to zero modulo $G$, which is
a Janet basis for $\idealdiffop$ with respect to $<$.
\end{remark}

\section{Module-theoretic constructions}\label{section_modules}

The techniques of Sect.~\ref{section_Ore} can be extended
to effectively deal with finitely presented left (and right) $\ringdiffop$-modules
and module homomorphisms between them.

\medskip

Let $\ringdiffop$ be as in the previous section and $q \in \N$. We choose the standard basis $e_1$, \ldots, $e_q$ of
the free left $\ringdiffop$-module $\ringdiffop^{1 \times q}$
and define the set of monomials
\[
\Mon(\ringdiffop^{1 \times q}) = \{ \, f_1^{\alpha_1} \ldots f_r^{\alpha_r} \, \diffindet_1^{\beta_1} \ldots \diffindet_d^{\beta_d} \, e_k \mid
\alpha_i, \beta_j \in \Z_{\ge 0}, \,
k = 1, \ldots, q \, \}.
\]
Then every element of $\ringdiffop^{1 \times q}$ has a unique
representation as in \eqref{eq:repres_element},
where $\Mon(\ringdiffop)$ is replaced by $\Mon(\ringdiffop^{1 \times q})$.
By generalizing the notion of monomial ordering defined in Sect.~\ref{section_Ore} to total orderings on $\Mon(\ringdiffop^{1 \times q})$,
one can extend the reduction procedure described in Remark~\ref{rem:reduction} and indeed
any algorithm computing Gr\"obner or Janet bases for
left ideals of $\ringdiffop$ to one that computes
Gr\"obner or Janet bases for submodules $\ringdiffop^{1 \times p} A$ of $\ringdiffop^{1 \times q}$, where $A \in \ringdiffop^{p \times q}$.
In particular, membership to such a submodule can be decided by reduction,
and therefore, computations with residue classes in $\ringdiffop^{1 \times q} / \ringdiffop^{1 \times p} A$
can be performed effectively.

\medskip

We recall some relevant monomial orderings on $\ringdiffop^{1 \times q}$.

\begin{example}\label{ex:TOPandPOT}
A monomial ordering $\prec$ on $\ringdiffop$ can be extended to
monomial orderings $<$ on $\ringdiffop^{1 \times q}$
in different ways, for example, by defining
\[
m_1 \, e_k \, < \,
m_2 \, e_l \quad \iff \quad
\Big( \, m_1 \prec m_2 \quad \mbox{or} \quad \big( \, m_1 = m_2 \quad \mbox{and} \quad
k > l \, \big) \, \Big)
\]
(``term-over-position''), or by defining
\[
m_1 \, e_k \, < \,
m_2 \, e_l \qquad \iff \qquad
\Big( \, k > l \quad \mbox{or} \quad \big( \, k = l \quad \mbox{and} \quad m_1 \prec m_2 \, \big) \, \Big)
\]
(``position-over-term''),
where $m_1$, $m_2 \in \Mon(\ringdiffop)$ and $k$, $l \in \{ 1, \ldots, q \}$.
\end{example}

\begin{example}\label{ex:elimtuples}
Let $s \in \{ 1, \ldots, q-1 \}$ and
$\prec_1$, $\prec_2$ be monomial orderings on $\ringdiffop^{1 \times s}$ and $\ringdiffop^{1 \times (q-s)}$, 
with standard bases $e_1$, \ldots, $e_s$ and $e_{s+1}$, \ldots, $e_q$, respectively.
A \emph{monomial ordering $<$ on $\ringdiffop^{1 \times q}$ eliminating $e_1$, \ldots, $e_s$}
is defined by
\[
\begin{array}{rcl}
m_1 \, e_k \, < \,
m_2 \, e_l & \iff &
\Big( \, l \le s < k \quad \, \mbox{or}\\[1em]
& & \quad \big( \, k \le s \quad \mbox{and} \quad l \le s \quad \mbox{and} \quad
m_1 \, e_k \, \prec_1 \, m_2 \, e_l
\, \big) \quad \, \mbox{or}\\[1em]
& & \quad \big( \, k > s \quad \mbox{and} \quad l > s \quad \mbox{and} \quad
m_1 \, e_k \, \prec_2 \, m_2 \, e_l
\, \big) \, \Big)\,,
\end{array}
\]
where $m_1$, $m_2 \in \Mon(\ringdiffop)$ and $k$, $l \in \{ 1, \ldots, q \}$.
\end{example}

\begin{remark}\label{rem:kernel}
Let $\varphi\colon \ringdiffop^{1 \times a} \to \ringdiffop^{1 \times b}$
be a homomorphism of left $\ringdiffop$-modules, represented by a matrix $\matrixrepres \in \ringdiffop^{a \times b}$.
A Janet basis for the nullspace of $\varphi$ can be computed as follows.
Join the two standard bases of $\ringdiffop^{1 \times a}$ and $\ringdiffop^{1 \times b}$ to obtain the basis $e_1$, \ldots, $e_a$, $e_{a+1}$, \ldots, $e_{a+b}$
of $\ringdiffop^{1 \times a} \oplus \ringdiffop^{1 \times b} \cong \ringdiffop^{1 \times (a+b)}$.
Let $<$ be a monomial ordering on $\ringdiffop^{1 \times (a+b)}$ as defined
in Example~\ref{ex:elimtuples} for $q = a+b$, $s = a$ and certain $\prec_1$ and $\prec_2$, i.e.,
eliminating $e_1$, \ldots, $e_a$.
Then let $J_0$ be a Janet basis, with respect to $<$,
for the submodule of $\ringdiffop^{1 \times (a+b)}$ generated
by the rows of the matrix $(\matrixrepres \quad I_a) \in \ringdiffop^{a \times (b+a)}$, where
$I_a$ is the identity matrix.
Now $J := \{ \, w \in \ringdiffop^{1 \times a} \mid (0, w) \in J_0 \, \}$ is a
a Janet basis for the nullspace of $\varphi$ with respect to $\prec_2$ (cf.\ also
\cite[Ex.~3.10]{RobRecentProgress},
\cite[Ex.~3.1.27]{Robertz6}).
\end{remark}

\begin{remark}\label{rem:involution}
An involution $\theta: \ringdiffop \to \ringdiffop$ of $\ringdiffop$ allows to reduce computations with right $\ringdiffop$-modules to computations with left $\ringdiffop$-modules.
More precisely, if we have
$\theta(r_1 + r_2) = \theta(r_1) + \theta(r_2)$
and $\theta(r_1 \, r_2) = \theta(r_2) \, \theta(r_1)$
and $\theta(\theta(r)) = r$ for all $r_1$, $r_2$, $r \in \ringdiffop$, then
any right $\ringdiffop$-module $M$ is turned into a left $\ringdiffop$-module $\widetilde{M} := M$ (as abelian groups) via $r \, m := m \, \theta(r)$,
where $r \in \ringdiffop$, $m \in \widetilde{M}$,
and vice versa.
The involution $\theta$ is extended to matrices by
(cf.\ also \cite[Rem.~3.11]{RobRecentProgress})
\[
\theta(A) := (\theta((A^{tr})_{i,j}))_{1 \le i \le q, 1 \le j \le p} \in \ringdiffop^{q \times p},
\qquad \quad A \in \ringdiffop^{p \times q}.
\]
Since for $A \in \ringdiffop^{p \times q}$,
$B \in \ringdiffop^{q \times r}$ we have
$A \, B = 0$ if and only if $\theta(B) \, \theta(A) = 0$, the computation of nullspaces of homomorphisms of right $\ringdiffop$-modules is reduced to the situation described in Remark~\ref{rem:kernel}.
For 
$\ringdiffop$
introduced in Definition~\ref{de:R} we choose
\[
\theta: \ringdiffop \to \ringdiffop, \quad
\theta|_{\diffalg} := {\rm id}_{\diffalg}, \quad
\theta(\diffindet_j) := -\diffindet_j, \quad
j = 1, \ldots, d.
\]
\end{remark}

\section{Parametrizing boundary conditions}\label{section_boundary}

This section constructs parametrizations of functions satisfying certain boundary conditions, independent of the parametrization of differential equations.

We restrict ourselves to boundary conditions parametrized by analytic functions for two reasons.
First, this allows algebraic algorithms.
Second, due to the limiting behaviour of \glspl!{gp} when conditioning on more data points, closed sets of functions are preferable, see Theorem~\ref{thm_analytical_spectral}.
For approximate resp.\ asymptotic resp.\ partially unknown boundary conditions for \glspl!{gp} see \cite{solin2019know} resp.\ \cite{tan2018gaussian} resp.\ \cite{gulian2020gaussian}.
For a theoretic approach to endow RKHS with boundary information see \cite{nicholson2022kernel}.

Let again $\F=C^\infty(\domain,\R)$ with Fr\'echet topology from \eqref{eq_frechet} be the set of smooth functions on $\domain\subset\R^d$ compact, $\diffalg = \R\{ f_1, \ldots, f_r \}$ with analytic functions $f_i: \domain\to\R$, and let $\ringdiffop\supseteq\diffalg$ be the Ore extension of $\diffalg$.

This section is based on two theorems.
The first one describes closed modules satisfying a Nullstellensatz via  their Taylor expansion.
Denote by $T_p$ the Taylor series of a (vector or matrix of) smooth function(s) around a point $p\in D$.

\begin{theorem}[Whitney's Spectral Theorem; {\cite{whitney1948ideals}, \cite[V Theorem~1.3]{tougeron1972ideaux}}]\label{thm_whitney_spectral}
    An $\F$-module $M\le C^\infty(D,\R)^\ell$ has topological closure $\overline{M}=\bigcap_{p\in D} T^{-1}_p(T_p(M))$.
\end{theorem}

The second theorem specifies that analytic functions generate closed modules.

\begin{theorem}[{\cite[Theorem 4]{malgrange1959division}, \cite[VI Theorem~1.1]{tougeron1972ideaux}}]\label{thm_analytical_spectral}
    Let  $C$ be an $m\times n$-matrix of analytic functions on $D\subset \R^d$ and $\phi\in \left(C^\infty(D,\R)\right)^m$.
    Then there is a $\psi\in \left(C^\infty(D,\R)\right)^n$ with $\phi=C\cdot \psi$ if and only if for all $p\in D$ the $T_p(\phi)$ are an $\R[[x_1-p_1,\ldots,x_d-p_d]]$-linear combination of the columns of $T(C)$.    
\end{theorem}

\subsection{Boundary conditions for function values of single functions}

We begin parametrizing functions which are zero on an analytic set $M$, e.g.\ Dirichlet boundary conditions which prescribe values at the boundary $\partial D$.

We define boundaries $M\subseteq D$ implicitly via
\begin{align*}
    M=\V(I):=\left\{m\in D\,\middle|\, b(m)=0\mbox{ for all }b\in I\right\}\subseteq D\mbox{,}
\end{align*}
where $I\trianglelefteq \diffalg$ is an ideal of equations.
For any analytic set $M\subseteq D$ we have $M=\V(\I(M))$, where $\I(M)=\{b\in\F^\ell\mid b(m)=0\mbox{ for all }m\in M\}\subseteq \F$ is the (closed and radical) ideal of functions vanishing at $M$.
If $I$ is radical (it is automatically closed by Theorem~\ref{thm_analytical_spectral}, as generated by analytic functions), then $\I(\V(I))=I$.
Hence, any set of analytic function defined on $\domain$ which generates a radical ideal parametrizes functions vanishing at its zero set.
More formally:

\begin{proposition}\label{proposition_boundary}
	Let $B'\in K^{1\times\ell}$ be a row of analytic functions whose entries generate a radical $\F$-ideal $I=B'\F^{\ell}\le\F$ of smooth functions.
	Then, $I$ is the set $\left\{f\in\F \mid f_{|\V(I)}=0\right\}$ of smooth functions vanishing at $\V(I)$.
\end{proposition}
\begin{proof}
    The condition $f_{|\V(I)}=0$ restricts the zeroth order Taylor coefficients by homogeneous equations.
    All functions satisfying such restrictions are contained in the closure $\overline{I}$ of $I$ by Whitney's Spectral Theorem~\ref{thm_whitney_spectral}.
    The $\F$-module parametrization $I=B'\F^{\ell}$ uses analytic functions as generators, which ensures that the ideal $I$ is already equal to its closure $\overline{I}$ by Theorem~\ref{thm_analytical_spectral}.\qed
\end{proof}

We now compare constructions of rows $B'$ of functions in Proposition~\ref{proposition_boundary}.

\begin{example}\label{example_dirichlet}
	Functions $\F=C^\infty([0,1]^d,\R)$ with Dirichlet boundary conditions $f(\partial D)=0$ at the boundary of the domain $D=[0,1]^d$ are parametrized by
	\begin{align}\label{example_dirichlet_1}
	    B_1'=\begin{bmatrix}
	        \prod_{i=1}^dx_i(x_i-1)
	    \end{bmatrix}
	\end{align}
	over $K=\R\{x_1,\ldots,x_n\}=\R[x_1,\ldots,x_n]$, by
	\begin{align}\label{example_dirichlet_2}
	    B_2'=\begin{bmatrix}
	        1-\exp\left((-1)^{d+1}\cdot\frac{\prod_{i=1}^dx_i(x_i-1)}{\delta}\right)
	    \end{bmatrix}
	\end{align}
	over $K=\R\{\exp(x_1^2),\exp(x_1),x_1,\ldots,\exp(x_d^2),\exp(x_d),x_d\}$, or by\footnote{$B_3'$ is obtained as the product of the standard deviations obtained by conditioning $d$ one-dimensional squared exponential covariances to the data points $(0,0)$ and $(1,0)$.}
    \begin{align}\label{example_dirichlet_3}
    	B_3'=\begin{bmatrix}
            \sqrt{\prod_{i=1}^d\left(1+\frac{\exp\left(-\frac{x_i^2}{\delta}\right)-2\exp\left(-\frac{x_i^2-x_i+1}{\delta}\right)+\exp\left(-\frac{(x_i-1)^2}{\delta}\right)}{\exp\left(-\frac{1}{\delta}\right)-1}\right)}
        \end{bmatrix}
    \end{align}
    for any $\delta>0$.
    See {\cite[Section~3]{NoisyLinearOperatorEquationsGP}} for the special case $d=2$ in \eqref{example_dirichlet_1}.
    For practical differences of these formalizations of boundary conditions see Remark~\ref{remark_comparison_boundaries}.
\end{example}

Block diagonal matrices parametrize boundaries of a vector of $\ell>1$ functions.
Also, restrictions on sets with higher codimension can be defined.

\begin{example}
	\sloppy
	The following three matrices 
	$\begin{bmatrix} 1-\exp\left(-\frac{|x|}{\delta}\right) & 1-\exp\left(-\frac{|y|}{\delta}\right) \end{bmatrix}$,
	$\begin{bmatrix} 1-\exp\left(-\frac{\sqrt{x^2+y^2}}{\delta}\right) \end{bmatrix}$,
	and
	$\begin{bmatrix} 1-\exp\left(-\frac{x^2+y^2}{\delta}\right) \end{bmatrix}$
	parametrize functions $f\in\F=C^\infty(\R^3,\R)$ with $f(0,0,z)=0$.
	The last parametrization is analytic.
\end{example}

\subsection{Boundary conditions for derivatives and vectors}

Boundary conditions with vanishing derivatives can be constructed using multiplicities in the (no longer radical) ideal.
The proof of the following proposition again follows from Theorems~\ref{thm_whitney_spectral} and \ref{thm_analytical_spectral}, in a similar way to Proposition~\ref{proposition_boundary}.

\begin{proposition}\label{proposition_boundary_derivatives}
	Let $B'\in K^{\ell\times\ell'}$ be a matrix of analytic functions whose columns generate an $\F$-module $M=B'\F^{\ell'}\le \F^\ell$ of smooth functions.
    Then, 
    \begin{align*}
        \textstyle \left\{f\in\F^\ell\,\middle|\,\forall p\in D\ \exists a_i\in\R[[x_1-p_i,\ldots,x_d-p_d]]\ \forall 1\le i\le\ell: T_p(f_i)=\sum_{j=1}^{\ell'} T_p(b_{ij})a_j \right\}
    \end{align*}
    is the closed set of smooth functions sharing the same vanishing lower order Taylor coefficients as the columns of $B'$.
\end{proposition}

\begin{figure}[t]
	\centering
	%
  \newcommand{\scaleboundary}{0.1}

  \begin{tikzpicture}[scale=0.35]
    \begin{axis}[ymin=-3,ymax=3,xmin=-0.01,xmax=1.01]
    \addplot[name path=upper,domain=-0.001:1.001,samples=50,blue] {2};
    \addplot[name path=lower,domain=-0.001:1.001,samples=50,blue] {-2};
    \addplot [color=blue,fill=blue,fill opacity=0.05] fill between[of=lower and upper, soft clip={domain=-0.001:1.001} ];
    \addplot[name path=mean,domain=-0.001:1.001,blue]{0};
    \end{axis}
  \end{tikzpicture}
  \begin{tikzpicture}[scale=0.35]
    \begin{axis}[ymin=-3,ymax=3,xmin=-0.01,xmax=1.01]
    \addplot[name path=upper,domain=-0.001:1.001,samples=50,blue] {2*sqrt(x^2*(x-1)^2)};
    \addplot[name path=lower,domain=-0.001:1.001,samples=50,blue] {-2*sqrt(x^2*(x-1)^2)};
    \addplot [color=blue,fill=blue,fill opacity=0.05] fill between[of=lower and upper, soft clip={domain=-0.001:1.001} ];
    \addplot[name path=mean,domain=-0.001:1.001,blue]{0};
    \end{axis}
  \end{tikzpicture}
  \begin{tikzpicture}[scale=0.35]
    \begin{axis}[ymin=-3,ymax=3,xmin=-0.01,xmax=1.01]
    \addplot[name path=upper,domain=-0.001:1.001,samples=150,blue] {2*(-1+exp(100*x*(x-1)))*(-1+exp(100*x*(x-1)))};
    \addplot[name path=lower,domain=-0.001:1.001,samples=150,blue] {-2*(-1+exp(100*x*(x-1)))*(-1+exp(100*x*(x-1)))};
    \addplot [color=blue,fill=blue,fill opacity=0.05] fill between[of=lower and upper, soft clip={domain=-0.001:1.001} ];
    \addplot[name path=mean,domain=-0.001:1.001,blue]{0};
    \end{axis}
  \end{tikzpicture}
  \begin{tikzpicture}[scale=0.35]
    \begin{axis}[ymin=-3,ymax=3,xmin=-0.01,xmax=1.01]
    \addplot[name path=upper,domain=-0.001:1.001,samples=150,blue] {2*sqrt((-2*exp((-x^2+x-1)/(1/10)^2)+exp(-x^2/(1/10)^2)+exp(-(x-1)^2/(1/10)^2)+exp(-1/(1/10)^2)-1)/(exp(-1/(1/10)^2)-1))};
    \addplot[name path=lower,domain=-0.001:1.001,samples=150,blue] {-2*sqrt((-2*exp((-x^2+x-1)/(1/10)^2)+exp(-x^2/(1/10)^2)+exp(-(x-1)^2/(1/10)^2)+exp(-1/(1/10)^2)-1)/(exp(-1/(1/10)^2)-1))};
    \addplot [color=blue,fill=blue,fill opacity=0.05] fill between[of=lower and upper, soft clip={domain=-0.001:1.001} ];
    \addplot[name path=mean,domain=-0.001:1.001,blue]{0};
    \end{axis}
  \end{tikzpicture}
  \begin{tikzpicture}[scale=0.35]
    \begin{axis}[ymin=-3,ymax=3,xmin=-0.01,xmax=1.01]
    \addplot[name path=upper,domain=-0.001:1.001,samples=150,blue] {2*((-2*exp(10^2*(-x^2+x-1))+exp(-10^2*x^2)+exp(-10^2*(x-1)^2)+exp(-10^2)-1)/(exp(-10^2)-1))};
    \addplot[name path=lower,domain=-0.001:1.001,samples=150,blue] {-2*((-2*exp(10^2*(-x^2+x-1))+exp(-10^2*x^2)+exp(-10^2*(x-1)^2)+exp(-10^2)-1)/(exp(-10^2)-1))};
    \addplot [color=blue,fill=blue,fill opacity=0.05] fill between[of=lower and upper, soft clip={domain=-0.001:1.001} ];
    \addplot[name path=mean,domain=-0.001:1.001,blue]{0};
    \end{axis}
  \end{tikzpicture}
  \begin{tikzpicture}[scale=0.35]
    \begin{axis}[ymin=-3,ymax=3,xmin=-0.01,xmax=1.01]
    \addplot[name path=upper,domain=-0.001:1.001,samples=150,blue] {2*sqrt(((exp(x^4))^100*(exp(x^2))^100-(exp(x^3))^200)^2/((exp(x^4))^200*(exp(x^2))^200))};
    \addplot[name path=lower,domain=-0.001:1.001,samples=150,blue] {-2*sqrt(((exp(x^4))^100*(exp(x^2))^100-(exp(x^3))^200)^2/((exp(x^4))^200*(exp(x^2))^200))};
    \addplot [color=blue,fill=blue,fill opacity=0.05] fill between[of=lower and upper, soft clip={domain=-0.001:1.001} ];
    \addplot[name path=mean,domain=-0.001:1.001,blue]{0};
    \end{axis}
  \end{tikzpicture}%

	\caption{%
		\glspl!{gp}, represented by their mean function and two standard deviations.
		Upper left: a \gls!{gp} $g$ with mean zero and square exponential covariance function.
		Upper middle: pushforward of the \gls!{gp} $g$ by $x\cdot(x-1)$ has a strong global influence.
		Upper right resp.\ lower left: pushforward of the \gls!{gp} $g$ by \eqref{example_dirichlet_2} resp.\ \eqref{example_dirichlet_3}.
		Lower middle resp.\ right: pushforward of the \gls!{gp} $g$ by \eqref{example_dirichlet_better_3} resp.\ \eqref{example_dirichlet_better_4} set the function and its derivative to zero at the boundary.
		Set $\delta:=\frac{1}{100}$.
	}
	\label{figure_comparison_boundaries}
\end{figure}
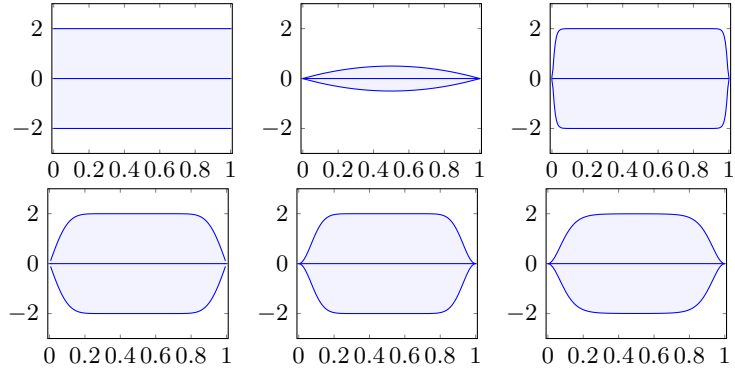

\begin{example}\label{example_boundary_derivative}
	Functions $\F=C^\infty([0,1]^d,\R)$ with Dirichlet boundary conditions $f(\partial D)=0$ and Neumann boundary condition $\frac{\partial f}{\partial n}(\partial D)=0$ for $n$ the normal to the boundary $\partial D$ of the domain $D=[0,1]^d$ are parametrized by
	\begin{align}\label{example_dirichlet_better_3}
	    B'=
	    \begin{bmatrix}
	        1-\exp\left((-1)^{d+1}\cdot\frac{\prod_{i=1}^dx_i^2(x_i-1)^2}{\delta}\right)
	    \end{bmatrix}\mbox{,}
	\end{align}
	constructed by squaring the exponent from the parametrization in \eqref{example_dirichlet_2}, or 
    \begin{align}\label{example_dirichlet_better_4}
    	B'=\begin{bmatrix}
            \prod_{i=1}^d\left(1+\frac{\exp\left(-\frac{x_i^2}{\delta}\right)-2\exp\left(-\frac{x_i^2-x_i+1}{\delta}\right)+\exp\left(-\frac{(x_i-1)^2}{\delta}\right)}{\exp\left(-\frac{1}{\delta}\right)-1}\right)
        \end{bmatrix}\mbox{,}
    \end{align}
    constructed by the squaring of the parametrization \eqref{example_dirichlet_3} for any $\delta>0$.
\end{example}

\begin{remark}\label{remark_comparison_boundaries}
    In applications, the non-polynomial parametrizations from Examples~\ref{example_dirichlet} and \ref{example_boundary_derivative} are more suitable.
    We demonstrate the effect by pushforward \glspl!{gp} obtained from these parametrizations in Figure~\ref{figure_comparison_boundaries}.
    
    The polynomial pushforward from Example~\ref{example_dirichlet} yields the variance $x^2\cdot(x-1)^2$, which strongly varies in the input interval $[0,1]$.
    The analytic pushforwards from Example~\ref{example_dirichlet} also set the variance to zero at the boundary, but quickly return to the original variance, and never exceed it.
    Even the speed of returning to the original variance can be controlled by changing the parameter $\delta$.
\end{remark}

\section{Examples}\label{section_examples}

Now, we intersect (Theorem~\ref{theorem_combining_parametrizations})  solution sets of differential equations and analytic boundary conditions (Sect.~\ref{section_boundary}) using the algorithms from Sects.~\ref{section_Ore} and \ref{section_modules}.

\begin{figure}[t]
	\centering
	\includegraphics{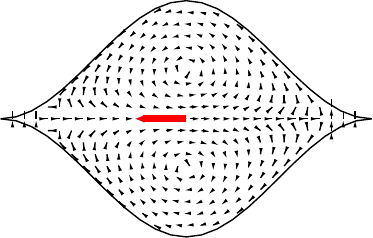}
	\includegraphics{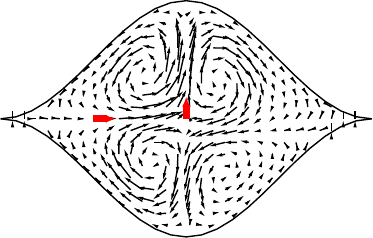}
	\caption{The mean fields of the \gls!{gp} for divergence-free fields in the interior of $y^2=\sin(x)^4$ from Example~\ref{example_double_drop}, which are conditioned on the data $(-1,0)$ at $(\frac{\pi}{2},0)$ (left) and on $(1,0)$ resp.\ $(0,1)$ at $(\frac{\pi}{4},0)$ resp.\ $(\frac{\pi}{2},0)$. The data is plotted artificially larger in red. The flow at the analytic boundary is zero.}
	\label{figure_double_drop}
\end{figure}

\begin{example}\label{example_double_drop}
    Consider divergence-free fields in the region in $\R^2$ bounded by $f:=y^2-\sin(x)^4$ for $x\in[0,\pi]$.
    Hence, consider $A=\begin{bmatrix} \partial_x & \partial_y \end{bmatrix}$, $B_1=\begin{bmatrix} \partial_y \\ -\partial_x \end{bmatrix}$ and $B_2=\begin{bmatrix} f & 0 \\ 0 & f \end{bmatrix}$.
    The Matrix $C=\begin{bmatrix} f^2 \\ \partial_y f \\ -\partial_x f \end{bmatrix}$ from Theorem~\ref{theorem_combining_parametrizations} yields the parametrization $\begin{bmatrix} \partial_yf^2 \\ -\partial_xf^2 \end{bmatrix}=\begin{bmatrix} f^2\partial_y+4\cdot f\cdot y \\ -f^2\partial_x+8\cdot f\cdot \sin(x)^3\cos(x) \end{bmatrix}$ and the push forward covariance
    \begin{align*}
        \tiny
        k\cdot f_1\cdot f_2\cdot
        \begin{bmatrix}
            f_1f_2+16y_1y_2+4\delta_y\cdot(f_1y_2-f_2y_1)-f_1f_2\delta_y^2 &
            (f_1y_1-f_1y_2+4y_1)\cdot(\delta_x\cdot f_2+\operatorname{sc}(x_2)) \\
            (f_2y_1-f_2y_2-4y_2)\cdot(\delta_x\cdot f_1+\operatorname{sc}(x_1) &
            (\operatorname{sc}(x_1)-\delta_x\cdot f_1)\cdot
            (\operatorname{sc}(x_2)+\delta_x\cdot f_2)
            -f_1f_2\cdot(2\delta_x^2-1)
        \end{bmatrix}
    \end{align*}
    of the squared exponential covariance function $k=\exp(-\frac{1}{2}((x_1-x_2)^2+(y_1-y_2)^2))$, 
    where $f_1=f(x_1,y_1)$, $f_2=f(x_2,y_2)$, $\delta_x=x_1-x_2$, $\delta_y=y_1-y_2$, and $\operatorname{sc}(x)=8\sin(x)^3\cos(x)$.
    For an illustration of this covariance see Figure~\ref{figure_double_drop}.
\end{example}




\begin{example}\label{example_compact_snake}
    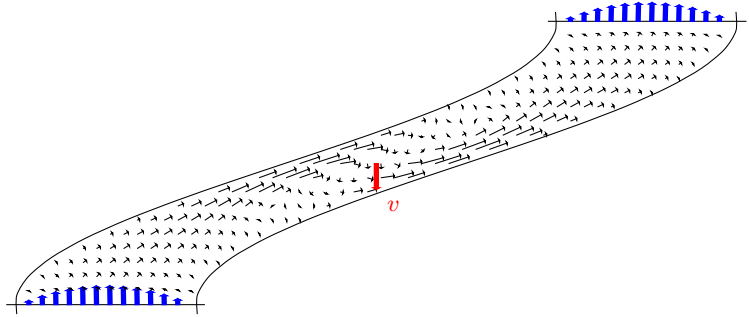
\begin{figure}[t]
        \centering
        %
  	\def\len{0.06}
	\def\heada{0.5}
	\def\headb{2}
	\begin{tikzpicture}[scale=1.2,domain=0:3.14159]
		
		\draw[color=black] (2.9,pi*0.5) -- (5.1,pi*0.5);
		\draw[color=black] (-3.1,-pi*0.5) -- (-0.9,-pi*0.5);
		
		\foreach \s in {0,2cm} {
			\draw [black, xshift=\s] plot [smooth] coordinates { 
					(-2.983565686, -1.675516082)
					(-2.983565686, -1.466076572)
					(-2.853169550, -1.256637062)
					(-2.598076211, -1.047197551)
					(-2.229434477, -0.8377580412)
					(-1.763355757, -0.6283185308)
					(-1.220209929, -0.4188790204)
					(-0.6237350727, -0.2094395103)
					(0., 0.)
					(0.6237350727, 0.2094395103)
					(1.220209929, 0.4188790204)
					(1.763355757, 0.6283185308)
					(2.229434477, 0.8377580412)
					(2.598076211, 1.047197551)
					(2.853169550, 1.256637062)
					(2.983565686, 1.466076572)
					(2.983565686, 1.675516082)
			};
		}
		\foreach \y in {-1.577, -1.42,...,1.7} {
			\pgfmathsetmacro{\siny}{sin(\y r)} 
			\pgfmathsetmacro{\cosy}{cos(\y r)} 
			\pgfmathsetmacro{\Pimy}{pi-2*\y}
			\pgfmathsetmacro{\sqrPiqmy}{\Pimy*\Pimy}
			\pgfmathsetmacro{\Pipy}{2*\y+pi}
			\pgfmathsetmacro{\sqrPiqpy}{\Pipy*\Pipy}
			\foreach \x in {-3,-2.85,...,5} {
				\pgfmathparse{(\x-3*\siny)*(\x-2-3*\siny) < -0.1 ? int(1) : int(0)}
				\ifthenelse{\pgfmathresult=1}{
					\pgfmathsetmacro{\lbound}{(3*\siny-\x)}
					\pgfmathsetmacro{\rbound}{\lbound+2}
					\pgfmathsetmacro{\sqrlqbound}{\lbound*\lbound}
					\pgfmathsetmacro{\sqrrqbound}{\rbound*\rbound}
					\pgfmathsetmacro{\vx}{
						((.1077928138*\rbound*\lbound*\sqrPiqmy*\sqrPiqpy*(\lbound+\rbound)*(\x-1)-.1385907618*\rbound*\lbound*\sqrPiqmy*\sqrPiqpy*(\lbound+\rbound)*\y)*\cosy+(-.1796546897e-1*\sqrrqbound*\sqrlqbound*\sqrPiqmy*\sqrPiqpy*\y+.7186187587e-1*\sqrrqbound*\sqrlqbound*\Pimy*\Pipy*(\Pimy-1.*\Pipy))*(\x-1)+.2309846030e-1*\sqrrqbound*\sqrlqbound*\sqrPiqmy*\sqrPiqpy*\y*\y-.9239384121e-1*\sqrrqbound*\sqrlqbound*\Pimy*\Pipy*(\Pimy-1.*\Pipy)*\y-.2309846030e-1*\sqrrqbound*\sqrlqbound*\sqrPiqmy*\sqrPiqpy)*exp(-0.5*(\x-1)*(\x-1)-0.5*\y*\y) -2.250000000*\cosy*\rbound*\lbound
					}
					\pgfmathsetmacro{\vy}{
						(.1796546897e-1*\sqrrqbound*\sqrlqbound*\sqrPiqmy*\sqrPiqpy*(\x-1)*(\x-1)+(-.2309846030e-1*\sqrrqbound*\sqrlqbound*\sqrPiqmy*\sqrPiqpy*\y+.3593093794e-1*\rbound*\lbound*\sqrPiqmy*\sqrPiqpy*(\lbound+\rbound))*(\x-1)-.4619692061e-1*\rbound*\lbound*\sqrPiqmy*\sqrPiqpy*(\lbound+\rbound)*\y-.1796546897e-1*\sqrrqbound*\sqrlqbound*\sqrPiqmy*\sqrPiqpy)*exp(-0.5*(\x-1)*(\x-1)-0.5*\y*\y) -.7500000000*\rbound*\lbound
					}
					\ifthenelse{\lengthtest{\y pt<-1.55 pt} \OR \lengthtest{\y pt>1.55 pt}}{
						\draw[-{Latex[length=2*\heada,width=2*\headb]},blue,line width =2]  (\x,\y) -- (\x+5*\len*\vx, \y+5*\len*\vy);
					}{
						\draw[-{Latex[length=\heada,width=\headb]},black]  (\x,\y) -- (\x+\len*\vx, \y+\len*\vy);
					}
				}{};
			}
		}
		\draw[-{Latex[length=2*\heada,width=2*\headb]},red,line width=2] (1.0,0.0) -- (1.0,-5*\len) node[below right]{\color{red}$v$};
	\end{tikzpicture}%

        \caption{
        	The mean fields of the \gls!{gp} for divergence-free fields from Example~\ref{example_compact_snake}, which are conditioned on $v=(0,-1)$ at $(0,1)$.
        	The flow at the left and right boundary is zero, at the bottom resp.\ top there is flow into resp.\ out of the region.
        	Both data point resp.\ inhomogeneous boundary conditions are plotted artificially larger in red resp.\ blue.
    	}
	    \label{figure_compact_snake}
    \end{figure}
    
    Consider divergence-free fields in the compact domain $D$ bounded by $-\frac{\pi}{2}\le y\le\frac{\pi}{2}$ and $3\sin(y)\le x\le3\sin(y)+2$.
    Hence, consider $A=\begin{bmatrix} \partial_x & \partial_y \end{bmatrix}$, $B_1=\begin{bmatrix} \partial_y \\ -\partial_x \end{bmatrix}$ and $B_2=\begin{bmatrix} f & 0 \\ 0 & f \end{bmatrix}$ for $f=(y-\frac{\pi}{2})\cdot(y+\frac{\pi}{2})\cdot(x-3\sin(y))\cdot(x-3\sin(y)-2)$.
    As the first entry in the column $C$ is $f^2$, such fields can be parametrized by $\begin{bmatrix} \partial_yf^2 \\ -\partial_xf^2 \end{bmatrix}$.
    Pushing forward the squared exponential covariance function yields a covariance too big to display.
    
    To encode non-zero boundary conditions we use a non-zero mean.
    Using the potential $p:=-\frac{1}{4}\cdot(3\sin(y)-x+3)\cdot(3\sin(y)-x)^2$ yields the divergence-free 
    \begin{align*}
    	\mu:=
    	\begin{bmatrix}
    		-\frac{9}{4} \cdot \cos(y) \cdot (3 \sin(y)-x+2) \cdot (3 \sin(y)-x)\\
    		-\frac{3}{4} \cdot (3 \sin(y)-x+2) \cdot (3 \sin(y)-x)
    	\end{bmatrix}
    	=
    	\begin{bmatrix}
    		\partial_y\\
    		-\partial_x
    	\end{bmatrix}
    	p
    \end{align*}
	satisfying the left and right boundaries 
    and non-zero flow through at top and bottom.
    The \gls!{gp} $\GP(\mu,k)$ hence models of divergence-free fields in $D$ with no flow on the sinoidal boundary left or right, but flow into $D$ from the bottom and out of $D$ at the top of the region.
    See Figure~\ref{figure_compact_snake} for a demonstration.
\end{example}



\section*{Acknowledgements}

The authors thank Andreas Besginow for discussions and the anonymous reviewers for helpful comments, both of which improved the contents of this paper.
\bibliographystyle{plain}

\end{document}